\def\year{2016}
\documentclass[letterpaper]{article}
\pdfoutput=1
\usepackage{aaai}
\usepackage{times}
\usepackage{helvet}
\usepackage{courier}
\usepackage{calc}
\usepackage{amssymb}
\usepackage{amstext}
\usepackage{amsmath}
\usepackage{amsthm}
\usepackage{natbib}%\usepackage[numbers,sort]{natbib}
\usepackage[small]{caption}
\usepackage{multirow}
\usepackage{array}
\usepackage[ruled,linesnumbered,vlined]{algorithm2e}
\usepackage{framed}

\DeclareMathOperator*{\argmax}{arg\,max}
\DeclareMathOperator*{\argmin}{arg\,min}

\newcommand{\uni}{\cup} % set union
\newcommand{\Prp}{\ensuremath{\mathcal{P}}}

\newtheorem{definition}{Definition}%[section] 
\newtheorem{lemma}{Lemma}%[section]
\newtheorem{theorem}{Theorem}%[section] 
\newtheorem{observation}{Observation}%[section]
\newtheorem{example}{Example}%[section]%}
\begin{document}
\title{Revising Incompletely Specified Convex Probabilistic Belief Bases}
%\author{Gavin Rens \and Thomas Meyer \and Giovanni Casini}
\author{Gavin Rens\\
CAIR\thanks{Centre for Artificial Intelligence Research},\\
University of KwaZulu-Natal,\\
School of Mathematics, Statistics and Comp.\ Sci. \\
CSIR Meraka, South Africa\\
Email: gavinrens@gmail.com
\And
Thomas Meyer\\
CAIR,\\
University of Cape Town,\\
Dept.\ of Comp.\ Sci. \\
CSIR Meraka, South Africa\\
Email: tmeyer@cs.uct.ac.za
\And
Giovanni Casini\\
University of Luxembourg,\\
Comp.\ Sci.\ and Communication Research Unit\\
Luxembourg\\
Email: giovanni.casini@uni.lu
}

\maketitle
\begin{abstract}
\begin{quote}
We propose a method for an agent to revise its incomplete probabilistic beliefs when a new piece of propositional information is observed.
In this work, an agent's beliefs are represented by a set of probabilistic formulae -- a belief base.
The method involves determining a \textit{representative} set of `boundary' probability distributions consistent with the current belief base, revising each of these probability distributions and then translating the revised information into a new belief base.
We use a version of Lewis Imaging as the revision operation.
The correctness of the approach is proved.
The expressivity of the belief bases under consideration are rather restricted, but has some applications.
We also discuss methods of belief base revision employing the notion of optimum entropy, and point out some of the benefits and difficulties in those methods.
Both the boundary distribution method and the optimum entropy method are reasonable, yet yield different results.
\end{quote}
\end{abstract}

\noindent
Suppose an agent represents its probabilistic knowledge with a set of statements; every statement says something about the probability of some features the agent is aware of.
Ideally, the agent would want to have enough information to, at least, identify one probability distribution over all the situations (worlds) it deems possible.
However, if the agent could not gather sufficient data or if it was not told or given sufficient information, it would not be able to pinpoint exactly one probability distribution.
An agent with this sort of ignorance, can be thought of as having beliefs compatible with a \textit{set of} distributions.
Now, this agent might need to revise its beliefs when new (non-probabilistic) information is received, even though the agent's beliefs do not characterize a \textit{particular} probability distribution over its current possible worlds.
% In other words, the agent's belief base may reflect its ignorance about the exact distribution over the worlds under consideration.

Several researchers argue that using a single probability distribution requires the agent to make unrealistically precise uncertainty distinctions \citep{gh98,v99a,yl08}.\footnote{See also the references in these cited papers concerning criticisms against traditional probability theory.}
``One widely-used approach to dealing with this has been to consider sets of probability measures as a way of modeling uncertainty,'' \citep{gh98}. However, simply applying standard
probabilistic conditioning to each of the measures/distributions in the set individually and then combining the results is also not recommended.
The framework presented in this paper proposes two ways to go from one `probabilistically incomplete' belief base to another when new information is acquired.

Both belief revision methods presented, essentially follow this process: From the original belief base, determine a relatively small set of belief states / probability distributions `compatible' with the belief base which is, in a sense, representative of the belief base. (We shall use the terms \textit{belief state}, \textit{probability distribution}, \textit{probability function} and \textit{distribution} interchangeably). Then revise every belief state in this representative set. Finally, induce a new, revised belief base from the revised representative set.

We shall present two approaches to determine the representative set of belief states from the current belief base: (i) The approach we focus on involves finding belief states which, in a sense, are at the boundaries of the constraints implied by the belief base. These `boundary belief states' can be thought of as drawing the outline of the convex space of beliefs. This outline is then revised to form a new outline shape, which can be translated into a new belief base. (ii) As a possible alternative approach, the representative set is a \textit{single} belief state which can be imagined to be at the center of the outline of the first approach. This `central' belief state is found by determining the one in the space of beliefs which is least biased or most entropic in terms of information theory \citep{j78,ct91}.

For approach (i) -- where the canonical set is the set of boundary belief states -- we shall prove that the revised canonical set characterizes the set of all belief states which would have resulted from revising all (including interior) belief states compatible with the original belief base.

The relevant background theory and notations are now introduced.

%\section{Preliminaries}\label{sec:preliminaries}
\bigskip
We shall work with classical propositional logic. Let $\Prp$ be the finite set of atomic propositional variables (\textit{atoms}, for short). Formally, a \textit{world} is a unique assignment of truth values to all the atoms in $\Prp$. There are thus $2^n$ conceivable worlds. An agent may consider some non-empty subset $W$ of the conceivable worlds called the possible worlds.
Often, in the exposition of this paper, a world will be referred to by its truth vector. For instance, if the vocabulary is placed in order $\langle q,r\rangle$ and $w_3\Vdash \lnot q\land r$, then $w_3$ may be referred to as $01$.\footnote{$w\Vdash\alpha$ is read `$w$ is a model for/satisfies $\alpha$'.}
Let $L$ be all propositional formulae which can be formed from $\Prp$ and the logical connectives $\land$ and $\lnot$, with $\top$ abbreviating tautology and $\bot$ abbreviating contradiction.

Let $\beta$ be a sentence in $L$. $[\beta]$ denotes the set of $\beta$-worlds, that is, the elements of $W$ satisfying $\beta$.
The worlds satisfying all sentences in a set of sentences $K$ are denoted by $[K]$.

We define the probabilistic language $L^\mathit{prob}= \{(\alpha)\bowtie x\mid \alpha\in L, \bowtie\in\{\leq,=,\geq\}, x\in[0,1]\}$.
%\bowtie\in\{\leq,<,=,>,\geq\}
Sentences with strict inequalities ($<,>$) are excluded from the language for now. Such sentences are more challenging to deal with and their inclusion is left for future work.
We propose a belief base (BB) to be a consistent (logically satisfiable) subset of $L^\mathit{prob}$. A BB specifies an agent's knowledge.

The basic semantic element of an agent's beliefs is a probability distribution or a \textit{belief state}
\[b=\{(w_1,p_1), (w_2,p_2), \ldots, (w_n,p_n)\},\] where $p_i$ is the probability that $w_i$ is the actual world in which the agent is.
% and every world $w_i$ appears at most once in $b$.
$\sum_{(w,p)\in b}p=1$.
We may also use $c$ to refer to a belief state.
%We shall not dictate whether the probabilities are subjective or objective (frequentist). We shall assume that the agent operates within one consistent probabilistic paradigm.
For parsimony, let $b=\langle p_1,\ldots, p_n \rangle$ be the probabilities that belief state $b$ assigns to $w_1,\ldots, w_n$ where $\langle w_1,w_2,w_3,w_4\rangle$ $=$ $\langle 11,10,01,00\rangle$, and $\langle w_1,w_2,\ldots,w_8\rangle$ $=$ $\langle 111,110,\ldots,000\rangle$.
Let $\Pi$ be the set of all belief states over $W$.

$b(\alpha)$ abbreviates $\sum_{w\in W,w\Vdash\alpha}b(w)$.
%$P(\beta\mid\alpha)$ abbreviates $\frac{P(\alpha\land\beta)}{P(\alpha)}$.
%
$b$ \textit{satisfies} formula $(\alpha)\bowtie x$ (denoted $b\Vdash(\alpha)\bowtie x$) iff $b(\alpha)\bowtie x$.
%$P$ \textit{satisfies} formula $(\beta\mid\alpha)= x$ (denoted $P\Vdash(\beta\mid\alpha)= x$) iff $P(\beta\mid\alpha)= x$.
If $B$ is a set of formulae, then $b$ \textit{satisfies} $B$ (denoted $b\Vdash B$) iff $\forall \gamma\in B$, $b\Vdash \gamma$.
If $B$ and $B'$ are sets of formulae, then $B$ \textit{entails} $B'$ (denoted $B\models B'$) iff for all $b\in\Pi$, $b\Vdash B'$ whenever $b\Vdash B$.
If $B\models \{\gamma\}$ then we simply write $B\models \gamma$. $B$ is logically equivalent to $B'$ (denoted $B\equiv B'$) iff $B\models B'$ and $B'\models B$.

Instead of an agent's beliefs being represented by a single belief state, a BB $B$ represents a \textit{set} of belief-states:
Let $\Pi^B:=\{b\in\Pi \mid b\Vdash B\}$.
A BB $B$ is \textit{satisfiable} (\textit{consistent}) iff $\Pi^B\neq\emptyset$.

\bigskip
The technique of \textit{Lewis imaging} for the revision of belief states, requires a notion of distance between worlds to be defined.
We use a pseudo-distance measure between worlds, as defined by \citet{lms01} and adopted by \citet{cnss14}.

We add a `faithfulness' condition, which we feel is lacking from the definition of \citet{lms01}: without this condition, a pseudo-distance measure would allow all worlds to have zero distance between them. \citet{b98a} mentions this condition, and we use his terminology: ``faithfulness''.
\begin{definition}
A pseudo-distance function $d~:~W \times W \to \mathbb{Z}$ satisfies the following four conditions: for all worlds $w, w', w'' \in W$,
\begin{enumerate}
\item $d(w, w') \geq 0$ (Non-negativity)
\item $d(w, w) = 0$ (Identity)
\item $d(w, w') = d(w', w)$ (Symmetry)
\item $d(w, w') + d(w', w'') \geq d(w, w'')$ (Triangular Inequality)
\item if $w\neq w'$, then $d(w, w') > 0$ (Faithfulness)
\end{enumerate}
\label{def:pseudo-func}
\end{definition}

Presently, the foundation theory, or paradigm, for studying belief change operations is commonly known as AGM theory \citep{agm85,g88}.
Typically, belief change (in a static world) can be categorized as expansion, revision or contraction, and is performed on a belief set, the set of sentences $K$ closed under logical consequence.
Expansion (denoted $+$) is the logical consequences of $K\uni\{\alpha\}$, where $\alpha$ is new information and $K$ is the current belief set.
Contraction of $\alpha$ is the removal of some sentences until $\alpha$ cannot be inferred from $K$.
It is the reduction of  beliefs.
Revision is when $\alpha$ is (possibly) inconsistent with $K$ and $K$ is (minimally) modified so that the new $K$ remains consistent and entails $\alpha$. In this view, when the new information is consistent with the original beliefs, expansion and revision are equivalent.

The next section presents a generalized imaging method for revising probabilistic belief states.
Then we describe the application of generalized imaging in our main contribution; revising boundary belief states instead of all belief states.
The subsequent section explain another approaches of revising our belief bases, which prepares us for discussions in the rest of the paper. The latter method finds a single representative belief state through maximum entropy inference.
Both the boundary belief state method and the maximum entropy method are reasonable, yet yield different results -- a seeming paradox is thus uncovered.
% Then we shall list six postulates to check how rational the main approach is. 
Then future possible directions of research are discussed.
We end with a section on the related work and the concluding section.

%\section{Revising via Boundary Belief States}
\section{Generalized Imaging}

It is not yet universally agreed what revision means in a probabilistic setting. One school of thought says that probabilistic expansion is equivalent to Bayesian conditioning. This is evidenced by Bayesian conditioning ($\mathsf{BC}$) being defined only when $b(\alpha)\neq0$, thus making $\mathsf{BC}$ expansion equivalent to $\mathsf{BC}$ revision.
In other words,
% using or Bayes' Rule to conditionalize and determine $p$, 
one could define expansion (restricted revision) to be 
%\[
%b\:\mathsf{BC}\:\alpha =\{(w,p)\mid w\in W, p= \frac{P(\alpha\mid w)b(w)}{\sum_{w'\in W}P(\alpha\mid w')b(w')}\},	
%\]
\[
b\:\mathsf{BC}\:\alpha =\{(w,p)\mid w\in W, p= b( w\mid\alpha),b(\alpha) \neq 0\}.
\]

To accommodate cases where $b(\alpha)=0$, that is, where $\alpha$ contradicts the agent's current beliefs and its beliefs need to be revised in the stronger sense, we shall make use of \textit{imaging}. Imaging was introduced by \citet{l76} as a means of revising a probability function. It has also been discussed in the work of, for instance, \citet{g88,dp93,cnss14,rm15b}.
%The following version of imaging must not be regarded as a fundamental part of the larger belief change framework presented here; it should be regarded as a place-holder or suggestion for the `revision-module' of the framework.
Informally, Lewis's original solution for accommodating contradicting evidence $\alpha$ is to move the probability of each world to its closest, $\alpha$-world. Lewis made the strong assumption that every world has a \textit{unique} closest $\alpha$-world. More general versions of imaging allows worlds to have \textit{several}, equally proximate, closest worlds.

\citet{g88} calls one of his generalizations of Lewis's imaging \textit{general imaging}. Our method is also a generalization. We thus refer to his as \textit{G\"ardenfors's general imaging} and to our method as \textit{generalized imaging} to distinguish them. It should be noted that all three these imaging methods are general revision methods and can be used in place of Bayesian conditioning for expansion. ``Thus imaging is a more general method of describing belief changes than conditionalization,'' \citep[p.~112]{g88}.

Let $\mathit{Min}(\alpha,w,d)$ be the set of $\alpha$-worlds closest to $w$ with respect to pseudo-distance $d$. Formally,
\begin{align*}
&\mathit{Min}(\alpha,w,d):=\\
&\qquad\{w'\in[\alpha]\mid \forall w''\in[\alpha], d(w',w)\leq d(w'',w)\},
\end{align*}
where $d(\cdot)$ is some pseudo-distance measure between worlds (e.g., Hamming or Dalal distance).
% It must also obey the faithfulness condition that for every world $w$, $d(w,w)<d(v,w)$ for all $v\neq w$.
\begin{example}
\label{ex:1}
Let the vocabulary be $\{q,r,s\}$.
Let $\alpha$ be $(q\land r)\lor(q\land\lnot r\land s)$. Suppose $d$ is Hamming distance. Then
\begin{align*}
&\mathit{Min}((q\land r)\lor(q\land\lnot r\land s),111,d)=\{111\}\\
&\mathit{Min}((q\land r)\lor(q\land\lnot r\land s),110,d)=\{110\}\\
&\mathit{Min}((q\land r)\lor(q\land\lnot r\land s),101,d)=\{101\}\\
&\mathit{Min}((q\land r)\lor(q\land\lnot r\land s),100,d)=\{110,101\}\\
&\mathit{Min}((q\land r)\lor(q\land\lnot r\land s),011,d)=\{111\}\\
&\mathit{Min}((q\land r)\lor(q\land\lnot r\land s),010,d)=\{110\}\\
&\mathit{Min}((q\land r)\lor(q\land\lnot r\land s),001,d)=\{101\}\\
&\mathit{Min}((q\land r)\lor(q\land\lnot r\land s),000,d)=\{110,101\}
\end{align*}
\hfill$\Box$
\end{example}

\begin{definition}[$\mathsf{GI}$]
\label{def:GI}
Then \emph{generalized imaging} (denoted $\mathsf{GI}$) is defined as
\begin{align*}
&b\:\mathsf{GI}\:\alpha :=\{(w,p)\mid w\in W, p=0 \mbox{ if } w\not\in[\alpha],\\
& \qquad \mbox{else }p=\sum_{\substack{w'\in W\\w\in\mathit{Min}(\alpha,w',d)}}b(w')/|\mathit{Min}(\alpha,w',d)|\}.
\end{align*}
\end{definition}
In words, $b\:\mathsf{GI}\:\alpha$ is the new belief state produced by taking the generalized image of $b$ with respect to $\alpha$.
Notice how the probability mass of non-$\alpha$-worlds is shifted to their closest $\alpha$-worlds. If a non-$\alpha$-world $w^\times$ with probability $p$ has $n$ closest $\alpha$-worlds (equally distant), then each of these closest $\alpha$-worlds gets $p/n$ mass from $w^\times$.

We define $b^\circ_\alpha:=b\circ\alpha$ so that we can write $b^\circ_\alpha(w)$, where $\circ$ is a revision operator.
\begin{example}
\label{ex:2}
Continuing on Example~\ref{ex:1}:
Let $b=\langle0,0.1,0,0.2,0,0.3,0,0.4\rangle$.

$(q\land r)\lor(q\land\lnot r\land s)$ is abbreviated as $\alpha$.

\bigskip
$b^\mathsf{GI}_\alpha(111) = \sum_{\substack{w'\in W\\111\in\mathit{Min}(\alpha,w',d)}}b(w')/|\mathit{Min}(\alpha,w',d)|$ $=$ $b(111)/|\mathit{Min}(\alpha,111,d)| + b(011)/|\mathit{Min}(\alpha,011,d)|$ $=$ $0/1 + 0/1$ $=0$.

\bigskip
$b^\mathsf{GI}_\alpha(110) = \sum_{\substack{w'\in W\\110\in\mathit{Min}(\alpha,w',d)}}b(w')/|\mathit{Min}(\alpha,w',d)|$ $=$ $b(110)/|\mathit{Min}(\alpha,110,d)| + b(100)/|\mathit{Min}(\alpha,100,d)| + b(010)/|\mathit{Min}(\alpha,010,d)| + b(000)/|\mathit{Min}(\alpha,000,d)|$ $=$ $0.1/1 + 0.2/2 + 0.3/1 + 0.4/2$ $=0.7$.

\bigskip
$b^\mathsf{GI}_\alpha(101) = \sum_{\substack{w'\in W\\101\in\mathit{Min}(\alpha,w',d)}}b(w')/|\mathit{Min}(\alpha,w',d)|$ $=$ $b(101)/|\mathit{Min}(\alpha,101,d)| + b(100)/|\mathit{Min}(\alpha,100,d)| + b(001)/|\mathit{Min}(\alpha,001,d)| + b(000)/|\mathit{Min}(\alpha,000,d)|$ $=$ $0/1 + 0.2/2 + 0/1 + 0.4/2$ $=0.3$.

\bigskip
And $b^\mathsf{GI}_\alpha(100) = b^\mathsf{GI}_\alpha(011) = b^\mathsf{GI}_\alpha(010) = b^\mathsf{GI}_\alpha(001) = b^\mathsf{GI}_\alpha(000) = 0$.
\hfill$\Box$
\end{example}

\section{Revision via $\mathsf{GI}$ and boundary belief states }

Perhaps the most obvious way to revise a given belief base (BB) $B$ is to revise every individual belief state in $\Pi^B$ and then induce a new BB from the set of revised belief states.
Formally, given observation $\alpha$, first determine a new belief state $b^\alpha$ for every $b\in \Pi^B$ via the defined revision operation:
\[
\Pi^{B^\alpha}=\{b^\alpha\in\Pi\mid b^\alpha=b\:\mathsf{GI}\:\alpha, \:b\in \Pi^B\}.
\]
If there is more than only a single belief state in $\Pi^B$, then $\Pi^B$ contains an infinite number of belief states.
Then how can one compute $\Pi^{B^\alpha}$? And how would one subsequently determine $B^\alpha$ from $\Pi^{B^\alpha}$?

In the rest of this section we shall present a finite method of determining $\Pi^{B^\alpha}$. What makes this method possible is the insight that $\Pi^B$ can be represented by a finite set of `boundary' belief states -- those belief states which, in a sense, represent the limits or the convex hull of $\Pi^B$. We shall prove that the set of revised boundary belief states defines $\Pi^{B^\alpha}$. Inducing $B^\alpha$ from $\Pi^{B^\alpha}$ is then relatively easy, as will be seen.

Let $W^\mathit{perm}$ be every permutation on the ordering of worlds in $W$. For instance, if $W=\{w_1,w_2,w_3,w_4\}$, then $W^\mathit{perm}=\{\langle w_1,w_2,w_3,w_4\rangle$, $\langle w_1,w_2,w_4,w_3\rangle$, $\langle w_1,w_3,w_2,w_4\rangle$, $\ldots$, $\langle w_4,w_3,w_2,w_1\rangle\}$.
Given an ordering $W^\#\in W^\mathit{perm}$, let $W^\#(i)$ be the $i$-th element of $W^\#$; for instance, $\langle w_4,w_3,w_2,w_1\rangle(2)=w_3$.
Suppose we are given a BB $B$.
We now define a function which, given a permutation of worlds, returns a belief state where worlds earlier in the ordering are assigned maximal probabilities according to the boundary values enforced by $B$.
\begin{definition}
$\mathit{MaxASAP}(B,W^\#)$ is the $b\in\Pi^B$ such that for $i=1,\ldots,|W|$, $\forall b'\in\Pi^B$, if $b'\neq b$, then $\sum_{j=1}^i b(W^\#(j)) \geq \sum_{k=1}^i b'(W^\#(k))$.
%\begin{align*}
%&\mathit{MaxASAP}(B,W^\#):= b\in\Pi^B\mid\mbox{ for }i=1,\ldots,|W|, \\
%&\qquad b(W^\#(i))=\max_{b'\in\Pi^B,b'(w_j)=b(w_j),j\in\{1,\ldots,i-1\}}b'(i)%\{\}
%\end{align*}
\end{definition}
\begin{example}
\label{ex:3}
Suppose the vocabulary is $\{q,r\}$ and $B_1 = \{(q)\geq0.6\}$.
Then, for instance, $\mathit{MaxASAP}(B_1,\langle 01$, $00$, $11$, $10\rangle)$ $=$ $\{(01,0.4)$, $(00,0)$, $(11,0.6)$, $(10,0)\}$ $=$ $\{(11,0.6)$, $(10,0)$, $(01,0.4)$, $(00,0)\}$.
\hfill$\Box$
\end{example}
\begin{definition}
\label{def:bound-bel-stts}
We define the boundary belief states of BB $B$ as the set
\begin{align*}
&\Pi^B_\mathit{bnd} := \{b\in \Pi^B\mid\\
&\qquad W^\#\in W^\mathit{perm}, b=\mathit{MaxASAP}(B,W^\#)\}
\end{align*}
\end{definition}
Note that $|\Pi^B_\mathit{bnd}|\leq|W^\mathit{perm}|$.
\begin{example}
\label{ex:4}
Suppose the vocabulary is $\{q,r\}$ and $B_1 = \{(q)\geq 0.6\}$.
Then
\begin{eqnarray*}
\Pi^{B_1}_\mathit{bnd} &=& \{\{(11,1.0),(10,0.0),(01,0.0),(00,0.0)\},\\ &&\{(11,0.0),(10,1.0),(01,0.0),(00,0.0)\},\\ &&\{(11,0.6),(10,0.0),(01,0.4),(00,0.0)\},\\ &&\{(11,0.6),(10,0.0),(01,0.0),(00,0.4)\},\\ &&\{(11,0.0),(10,0.6),(01,0.4),(00,0.0)\},\\ &&\{(11,0.0),(10,0.6),(01,0.0),(00,0.4)\}\}.
\end{eqnarray*}
\hfill$\Box$
\end{example}

Next, the revision operation is applied to every belief state in $\Pi^B_\mathit{bnd}$.
Let $(\Pi^B_\mathit{bnd})^\mathsf{GI}_\alpha := \{b'\in\Pi\mid b'=b^\mathsf{GI}_\alpha, \:b\in\Pi^B_\mathit{bnd}\}$.
\begin{example}
\label{ex:5}
Suppose the vocabulary is $\{q,r\}$ and $B_1 = \{(q)\geq 0.6\}$.
Let $\alpha$ be $(q\land\lnot r)\lor(\lnot q\land r)$.
Then
\begin{eqnarray*}
(\Pi^{B_1}_\mathit{bnd})^\mathsf{GI}_\alpha &=& \{\{(11,0.0),(10,0.5),(01,0.5),(00,0.0)\},\\ &&\{(11,0.0),(10,1.0),(01,0.0),(00,0.0)\},\\ &&\{(11,0.0),(10,0.3),(01,0.7),(00,0.0)\},\\ &&\{(11,0.0),(10,0.6),(01,0.4),(00,0.0)\},\\ &&\{(11,0.0),(10,0.8),(01,0.2),(00,0.0)\}\}.
\end{eqnarray*}
(Two revision operations produce $\{(11,0),(10,0.5),(01,0.5),(00,0)\}$.)
\hfill$\Box$
\end{example}

To induce the new BB $B_\mathit{bnd}^\alpha$ from $(\Pi^B_\mathit{bnd})^\mathsf{GI}_\alpha$, the following procedure is executed.
For every possible world, the procedure adds a sentence enforcing the upper (resp., lower) probability limit of the world, with respect to all the revised boundary belief states. Trivial limits are excepted.
\begin{framed}
For every $w\in W$, $(\phi_w)\leq \overline{y}
\in B^\alpha$, where $\overline{y}=\max_{b\in (\Pi^B_\mathit{bnd})^\mathsf{GI}_\alpha}b(w)$, except when $\overline{y}=1$, and $(\phi_w)\geq \underline{y}\in B^\alpha$, where $\underline{y}=\min_{b\in (\Pi^B_\mathit{bnd})^\mathsf{GI}_\alpha}b(w)$, except when $\underline{y}=0$.
\end{framed}
The intention is that the procedure specifies $B^\alpha$ to represent the upper and lower probability envelopes of the set of revised boundary belief states -- $B^\alpha$ thus defines the entire revised belief state space (cf.\ Theorem~\ref{th:main}).
\begin{example}
\label{ex:6}
Continuing Example~\ref{ex:5},
using the translation procedure just above, we see that $B^\alpha_{1 \mathit{bnd}} = \{(\phi_{11})\leq0$, $(\phi_{10})\geq0.3$, $(\phi_{01})\leq0.7$, $(\phi_{00})\leq0.0\}$.

Note that if we let $B'=\{((q\land\lnot r)\lor(\lnot q\land r))=1$, $(q\land\lnot r)\geq0.3\}$, then $\Pi^{B'}=\Pi^{B^\alpha_{1 \mathit{bnd}}}$.
\hfill$\Box$
\end{example}

\begin{example}
\label{ex:7}
%Note that sentences of the form $(\beta)=x\in B$ are ignored in Step 2.
Suppose the vocabulary is $\{q,r\}$ and $B_2 = \{(\lnot q\land\lnot r)=0.1\}$.
Let $\alpha$ be $\lnot q$.
Then
\begin{eqnarray*}
\Pi^{B_2}_\mathit{bnd} &=& \{\{(11,0.9),(10,0),(01,0),(00,0.1)\},\\
&&\{(11,0),(10,0.9),(01,0),(00,0.1)\},\\
&&\{(11,0),(10,0),(01,0.9),(00,0.1)\}\},
\end{eqnarray*}
\begin{eqnarray*}
(\Pi^{B_2}_\mathit{bnd})^\mathsf{GI}_\alpha &=& \{\{(11,0),(10,0),(01,0.9),(00,0.1)\},\\
&&\{(11,0),(10,0),(01,0),(00,1)\}\}\mbox{ and}
\end{eqnarray*}
$B^\alpha_{2 \mathit{bnd}} = \{(\phi_{11})\leq0$, $(\phi_{10})\leq0$, $(\phi_{01})\leq0.9$, $(\phi_{00})\geq0.1\}$.

Note that if we let $B'=\{(\lnot q)=1$, $(\lnot q\land r)\leq0.9\}$, then $\Pi^{B'}=\Pi^{B^\alpha_{2 \mathit{bnd}}}$.
\hfill$\Box$
\end{example}

Let $W^{\mathit{Min}(\alpha,d)}$ be a partition of $W$ such that $\{w^i_1, \ldots, w^i_{ni}\}$ is a block in $W^{\mathit{Min}(\alpha,d)}$ iff $|\mathit{Min}(\alpha,w^i_1,d)| = \cdots = |\mathit{Min}(\alpha,w^i_{ni},d)|$.
Denote an element of block $\{w^i_1, \ldots, w^i_{ni}\}$ as $w^i$, and the block of which $w^i$ is an element as $[w^i]$.
Let $i=|\mathit{Min}(\alpha,w^i,d)|$, in other words, the superscript in $w^i$ indicates the size of $\mathit{Min}(\alpha,w^i,d)$.
Let $m := \max_{w\in W}|\mathit{Min}(\alpha,w,d)|$.

% Example 8 -- dealing with strict inequalities -- comes here.

\begin{observation}
\label{obs:1}
Let $\delta_1, \delta_2, \ldots, \delta_m$ be positive integers such that $i<j$ iff $\delta_i < \delta_j$. 
Let $\nu_1, \nu_2, \ldots, \nu_m$ be values in $[0,1]$ such that $\sum_{k=1}^m \nu_k = 1$.
Associate with every $\nu_i$ a maximum value it is allowed to take: $\mathit{most}(\nu_i)$.
For every $\nu_i$, we define the \emph{assignment value}
\[
\mathit{av}(\nu_i) := \left\{
\begin{array}{rl}
\mathit{most}(\nu_i) & \text{if } \sum_{k=1}^i\leq1\\
1-\sum_{k=1}^{i-1} & \text{otherwise}
\end{array} \right.
\]
Determine first $\mathit{av}(\nu_1)$, then $\mathit{av}(\nu_2)$ and so on.
Then
\[
\frac{\mathit{av}(\nu_1)}{\delta_1} + \cdots + \frac{\mathit{av}(\nu_m)}{\delta_m}>\frac{\nu'_1}{\delta_1} + \cdots + \frac{\nu'_m}{\delta_m}
\]
whenever $\nu'_i\neq\mathit{av}(\nu_i)$ for some $i$.
\hfill$\Box$
\end{observation}
For instance,
let $\delta_1=1$, $\delta_2=2$, $\delta_3=3$, $\delta_4=4$. Let $\mathit{most}(\nu_1)=0.5$, $\mathit{most}(\nu_2)=0.3$, $\mathit{most}(\nu_3)=0.2$, $\mathit{most}(\nu_4)=0.3$.
Then $\mathit{av}(\nu_1)=0.5$, $\mathit{av}(\nu_2)=0.3$, $\mathit{av}(\nu_3)=0.2$, $\mathit{av}(\nu_4)=0$ and
\[\frac{0.5}{1}+\frac{0.3}{2}+\frac{0.2}{3}+\frac{0}{4}=0.716.\]
But
\[\frac{0.49}{1}+\frac{0.3}{2}+\frac{0.2}{3}+\frac{0.01}{4}=0.709.\]
And
\[\frac{0.5}{1}+\frac{0.29}{2}+\frac{0.2}{3}+\frac{0.01}{4}=0.714.\]

Lemma~\ref{lm:13} essentially says that the belief state in $\Pi^B$ which causes a revised belief state to have a maximal value at world $w$ (w.r.t.\ all belief states in $\Pi^B$), will be in $\Pi^B_\mathit{bnd}$.
\begin{lemma}
\label{lm:13}
For all $w\in W$, $\argmax_{b_X\in\Pi^B}\sum_{\substack{w'\in W\\w\in\mathit{Min}(\alpha,w',d)}}b_X(w')/|\mathit{Min}(\alpha,w',d)|$ is in $\Pi^B_\mathit{bnd}$.
\end{lemma}
\begin{proof}
Note that
\[\sum_{\substack{w'\in W\\w\in\mathit{Min}(\alpha,w',d)}}b(w')/|\mathit{Min}(\alpha,w',d)|\]
can be written in the form
\begin{equation*}
\label{eq:grouped-form1}
\frac{\sum_{\substack{w'\in[w^1]\\w\in\mathit{Min}(\alpha,w',d)}}b(w')}{1} + \cdots + \frac{\sum_{\substack{w'\in[w^m]\\w\in\mathit{Min}(\alpha,w',d)}}b(w')}{m}.
\end{equation*}
Observe that there must be a $W^\#\in W^\mathit{perm}$ such that $W^\# = \langle w^1_1, \ldots, w^1_{n1},\ldots,w^m_1, \ldots, w^m_{nm}\rangle$.
Then by the definition of the set of boundary belief states (Def.~\ref{def:bound-bel-stts}), 
$\mathit{MaxASAP}(B,W^\#)$ will assign maximal probability mass to $[w^1] = \{w^1_1, \ldots, w^1_{n1}\}$, then to $[w^2] = \{w^2_1, \ldots, w^m_{n2}\}$ and so on.

That is, by Observation~\ref{obs:1}, for some $b_x \in \Pi^B_\mathit{bnd}$,
$b_x(w) = max_{b_X\in\Pi^B}\sum_{\substack{w'\in W\\w\in\mathit{Min}(\alpha,w',d)}}b_X(w')/|\mathit{Min}(\alpha,w',d)|$ for all $w\in W$.
Therefore, $\argmax_{b_X\in\Pi^B}\sum_{\substack{w'\in W\\w\in\mathit{Min}(\alpha,w',d)}}b_X(w')/|\mathit{Min}(\alpha,w',d)|$ is in $\Pi^B_\mathit{bnd}$.
\end{proof}

Let

\begin{tabular}{ll}
$\overline{x}^w := \max_{b\in\Pi^B_\mathit{bnd}}b(w)$ & $\overline{X}^w := \max_{b\in\Pi^B}b(w)$\\[3mm]
$\overline{y}^w := \max_{b\in(\Pi^B_\mathit{bnd})^\mathit{GI}_\alpha}b(w)$ & $\overline{Y}^w := \max_{b\in(\Pi^B)^\mathit{GI}_\alpha}b(w)$\\[3mm]
$\underline{x}^w := \min_{b\in\Pi^B_\mathit{bnd}}b(w)$ & $\underline{X}^w := \min_{b\in\Pi^B}b(w)$\\[3mm]
$\underline{y}^w := \min_{b\in(\Pi^B_\mathit{bnd})^\mathit{GI}_\alpha}b(w)$ & $\underline{Y}^w := \min_{b\in(\Pi^B)^\mathit{GI}_\alpha}b(w)$
\end{tabular}

\medskip
Lemma~\ref{lm:lim-before-equals-lim-after} states that for every world, the upper/lower probability of the world with respect to $\Pi^B_\mathit{bnd}$ is equal to the upper/lower probability of the world with respect to $\Pi^B$.
The proof requires Observation~\ref{obs:1} and Lemma~\ref{lm:13}.
\begin{lemma}
\label{lm:lim-before-equals-lim-after}
For all $w\in W$, $\overline{y}^w =\overline{Y}^w$ and $\underline{y}^w =\underline{Y}^w$.
\end{lemma}
\begin{proof}
Note that if $w\not\in[\alpha]$, then $\overline{y}^w =\overline{Y}^w = 0$ and $\underline{y}^w =\underline{Y}^w = 0$.

We now consider the cases where $w\in[\alpha]$.

\[\overline{y}^w =\overline{Y}^w\]
iff
\[\max_{b\in(\Pi^B_\mathit{bnd})}b(w)= \max_{b\in(\Pi^B)}b(w)\] 
iff
\begin{align*}
&\max_{b_x\in\Pi^B_\mathit{bnd}}\sum_{\substack{w'\in W\\w\in\mathit{Min}(\alpha,w',d)}}b_x(w')/|\mathit{Min}(\alpha,w',d)|\\
&\quad = \max_{b_X\in\Pi^B}\sum_{\substack{w'\in W\\w\in\mathit{Min}(\alpha,w',d)}}b_X(w')/|\mathit{Min}(\alpha,w',d)|
\end{align*}
if

$\overline{b}_x(w)=\overline{b}_X(w)$, where
\[\overline{b}_x(w):=\max_{b_x\in\Pi^B_\mathit{bnd}}\sum_{\substack{w'\in W\\w\in\mathit{Min}(\alpha,w',d)}}b_x(w')/|\mathit{Min}(\alpha,w',d)|\]
and
\[\overline{b}_X(w):=\max_{b_X\in\Pi^B}\sum_{\substack{w'\in W\\w\in\mathit{Min}(\alpha,w',d)}}b_X(w')/|\mathit{Min}(\alpha,w',d)|.\]

Note that
\[\sum_{\substack{w'\in W\\w\in\mathit{Min}(\alpha,w',d)}}b(w')/|\mathit{Min}(\alpha,w',d)|\]
can be written in the form
\begin{equation*}
\label{eq:grouped-form2}
\frac{\sum_{\substack{w'\in[w^1]\\w\in\mathit{Min}(\alpha,w',d)}}b(w')}{1} + \cdots + \frac{\sum_{\substack{w'\in[w^m]\\w\in\mathit{Min}(\alpha,w',d)}}b(w')}{m}.
\end{equation*}
Then by Observation~\ref{lm:13}, $\overline{b}_X(w)$ is in $\Pi^B_\mathit{bnd}$. And also by Lemma~\ref{lm:13}, the belief state in $\Pi^B_\mathit{bnd}$ identified by $\overline{b}_X(w)$ must be the one which maximizes
\[\sum_{\substack{w'\in W\\w\in\mathit{Min}(\alpha,w',d)}}b_x(w')/|\mathit{Min}(\alpha,w',d)|,\] where $b_x\in\Pi^B_\mathit{bnd}$. That is, $\overline{b}_x=\overline{b}_X$.

With a symmetrical argument, it can be shown that $\underline{y}^w =\underline{Y}^w$.
\end{proof}

\medskip
In intuitive language, the following theorem says that the BB determined through the method of revising boundary belief states captures exactly the same beliefs and ignorance as the belief states in $\Pi^B$ which have been revised.
This correspondence relies on the fact that the upper and lower probability envelopes of $\Pi^B$ can be induce from $\Pi^B_\mathit{bnd}$, which is what Lemma~\ref{lm:lim-before-equals-lim-after} states.
\begin{theorem}
\label{th:main}
Let $(\Pi^B)^\mathsf{GI}_\alpha := \{b^\mathsf{GI}_\alpha\in\Pi\mid b\in\Pi^B\}$.
Let $B^\alpha_\mathit{bnd}$ be the BB induced from $(\Pi^B_\mathit{bnd})^\mathsf{GI}_\alpha$.
Then $\Pi^{B^\alpha_\mathit{bnd}} = (\Pi^B)^\mathsf{GI}_\alpha$.
\end{theorem}
\begin{proof}
We show that $\forall b'\!\in\!\Pi$, $b'\!\in\!\Pi^{B^\alpha_\mathit{bnd}} \iff b'\!\in\!(\Pi^B)^\mathsf{GI}_\alpha$.

($\Rightarrow$)
$b'\!\in\!\Pi^{B^\alpha_\mathit{bnd}}$
implies $\forall w\in W$, $\underline{y}^w \leq b'(w)\leq\overline{y}^w$ (by definition of $B^\alpha_\mathit{bnd}$).
Lemma~\ref{lm:lim-before-equals-lim-after} states that for all $w\in W$, $\overline{y}^w =\overline{Y}^w$ and $\underline{y}^w =\underline{Y}^w$.
Hence,
$\forall w\in W$, $\underline{Y}^w \leq b'(w)\leq\overline{Y}^w$
Therefore,
$b'(w)\in(\Pi^B)^\mathsf{GI}_\alpha$.

($\Leftarrow$)
$b'(w)\in(\Pi^B)^\mathsf{GI}_\alpha$
implies $\forall w\in W$, $\underline{Y}^w \leq b'(w)\leq\overline{Y}^w$.
Hence, by Lemma~\ref{lm:lim-before-equals-lim-after},
$\forall w\in W$, $\underline{y}^w \leq b'(w)\leq\overline{y}^w$.
Therefore, by definition of $B^\alpha_\mathit{bnd}$, $b'\!\in\!\Pi^{B^\alpha_\mathit{bnd}}$.
\end{proof}

\section{Revising via a Representative Belief State}

Another approach to the revision of a belief base (BB) is to determine a representative of $\Pi^B$ (call it $b_\mathit{rep}$), change the representative belief state via the the defined revision operation and then induce a new BB from the revised representative belief state. Selecting a representative probability function from a family of such functions is not new \citep[e.g.]{gmp90,p94c}.
More formally, given observation $\alpha$, first determine $b_\mathit{rep}\in\Pi^B$, then compute its revision $b_\mathit{rep}^\alpha$, and finally induce $B^\alpha$ from $b_\mathit{rep}^\alpha$.

We shall represent $\Pi^B$ (and thus $B$) by the single `least biased' belief state, that is, the belief state in $\Pi^B$ with \textit{highest entropy}:
\begin{definition}[Shannon Entropy]
\label{def:shannon}
\[H(b):= -\sum_{w\in W}b(w)\ln b(w),\]
where $b$ is a belief state.
\end{definition}
\begin{definition}[Maximum Entropy]
\label{def:max-ent}
Traditionally, given some set of distributions $\Pi$, the most entropic distribution in $\Pi$ is defined as
\[
b^H := \argmax_{b\in\Pi}H(b).
\]
\end{definition}

Suppose $B_2 = \{(\lnot q\land \lnot r)=0.1\}$. Then the belief state $b\in\Pi^{B_2}$ satisfying the constraints posed by $B_2$ for which $H(b)$ is maximized is $b_\mathit{rep} = b^H = \langle0.3$, $0.3$, $0.3$, $0.1\rangle$.

The above distribution can be found directly by applying the principle of maximum entropy: The true belief state is estimated to be the one consistent with known constraints, but is otherwise as unbiased as possible, or ``Given no other knowledge, assume that everything is as random as possible. That is, the probabilities are distributed as uniformly as possible consistent with the available information,'' \citep{pm10}.
Obviously world 00 must be assigned probability 0.1. And the remaining 0.9 probability mass should be uniformly spread across the other three worlds.

Applying $\mathsf{GI}$ to $b_\mathit{rep}$ on evidence $\lnot q$ results in $b_\mathit{rep}^{\lnot q} = \langle 0$, $0$, $0.6$, $0.4\rangle$.

\begin{example}
Suppose the vocabulary is $\{q,r\}$, $B_1 = \{(q)\geq 0.6\}$ and $\alpha$ is $(q\land\lnot r)\lor(\lnot q\land r)$.
Then $b_\mathit{rep} = \argmax_{b\in\Pi^{B_1}}H(b) = \langle0.3$, $0.3$, $0.2$, $0.2 \rangle$.
Applying $\mathsf{GI}$ to $b_\mathit{rep}$ on $\alpha$ results in $b_\mathit{rep}^\alpha = \langle 0$, $0.61$, $0.39$, $0\rangle$.
$b_\mathit{rep}^\alpha$ can be translated into $B_{1 \mathit{rep}}^\alpha$ as $\{(q\land\lnot r)=0.61, (\lnot q\land r)=0.39\}$.
\hfill$\Box$
\end{example}
Still using $\alpha = (q\land\lnot r)\lor(\lnot q\land r)$,
notice that $\Pi^{B_{1 \mathit{rep}}^\alpha} \neq \Pi^{B_{1 \mathit{bnd}}^\alpha}$.
But how different \textit{are} $B_{1 \mathit{rep}}^\alpha$ $=$ $\{(q\land\lnot r)=0.61, (\lnot q\land r)=0.39\}$ and $B_{1 \mathit{bnd}}^\alpha$ $=$ $\{(q\land r)\leq0$, $(q\land\lnot r)\geq0.3$, $(\lnot q\land r)\leq0.7$, $(\lnot q\land \lnot r)\leq0.0\}$? Perhaps one should ask, how different $B_{1 \mathit{rep}}^\alpha$ is from the representative of $B_{1 \mathit{bnd}}^\alpha$: The least biased belief state satisfying $B_{1 \mathit{bnd}}^\alpha$ is $\langle 0,0.5,0.5,0\rangle$. That is, How different are $\langle 0,0.61,0.39,0\rangle$ and $\langle 0,0.5,0.5,0\rangle$?

In the case of $B_2$, we could compare $B^{\lnot q}_{2 \mathit{bnd}} = \{(\phi_{11})\leq0$, $(\phi_{10})\leq0$, $(\phi_{01})\leq0.9$, $(\phi_{00})\geq0.1\}$  with $b^{\lnot q}_\mathit{rep} = \langle0$, $0$, $0.6$, $0.4\rangle$.
Or if we take the least biased belief state satisfying $B^{\lnot q}_{2 \mathit{bnd}}$, we can compare $\langle0$, $0$, $0.5$, $0.5\rangle$ with $\langle0$, $0$, $0.6$, $0.4\rangle$.

It has been extensively argued \citep{j78,sj80,pv97} that maximum entropy is a reasonable inference mechanism, if not the most reasonable one (w.r.t.\ probability constraints).
And in the sense that the boundary belief states method requires no compression / information loss, it also seems like a very reasonable inference mechanism for revising BBs as defined here.
% It thus seems like we have discovered a paradox. 
Resolving this misalignment in the results of the two methods is an obvious task for future research.

%Our focus in this paper is to introduce the method of revision via boundary belief states. Hence, we shall not investigate the `representative belief state' method further here.

\section{Future Directions}

%Other belief change operations might be applied to the methods presented, instead of imaging. It seems to be a simple matter of applying the other operation whenever revision would have been. Future work may confirm this.

Some important aspects still missing from our framework are the representation of conditional probabilistic information such as is done in the work of \citeauthor{k08}, and the association of information with its level of entrenchment.
%[Remove following par.?]
%On the former point, much of human reasoning is conditional, in the sense that we know the likelihood of an event (proposition) \textit{given} the knowledge of another event.
On the latter point, when one talks about probabilities or likelihoods, if one were to take a frequentist perspective, information observed more (less) often should become more (less) entrenched. Or, without considering observation frequencies, an agent could be designed to have, say, one or two sets of deeply entrenched background knowledge (e.g., domain constraints) which does not change or is more immune to change than `regular' knowledge.

Given that we have found that the belief base resulting from revising via the boundary-belief-states approach differs from the belief base resulting from revising via the representative-belief-state approach, the question arises, When is it appropriate to use a representative belief state defined as the most entropic belief state of a given set $\Pi^B$? This is an important question, especially due to the popularity of employing the Maximum Entropy principle in cases of undespecified probabilistic knowledge \citep{j78,gmp90,h91,v99a,k01a,kr04} and the principle's well-behavedness \citep{sj80,p94c,k98}.

\citet{km91} modified the eight AGM belief revision postulates \citep{agm85} to the following six (written in the notation of this paper), where $*$ is some revision operator.\footnote{In these postulates, it is sometimes necessary to write an observation $\alpha$ as a BB, i.e., as $\{(\alpha)=1\}$ -- in the present framework, observations are regarded as certain.}
\begin{itemize}
\itemsep=0pt
\item $B_*^\alpha\models(\alpha)=1$.
\item If $B\uni \{(\alpha)=1\}$ is satisfiable, then  $B_*^\alpha \equiv B\uni \{(\alpha)=1\}$.
\item If $(\alpha)=1$ is satisfiable, then $B_*^\alpha$ is also satisfiable.
\item If $\alpha\equiv\beta$, then $B_*^\alpha \equiv B_*^\beta$.
\item $B_*^\alpha\uni \{(\beta)=1\} \models B_*^{\alpha\land\beta}$.
\item If $B_*^\alpha\uni \{(\beta)=1\}$ is satisfiable, then $B_*^{\alpha\land\beta} \models B_*^\alpha\uni \{(\beta)=1\}$.
\end{itemize}
Testing the various revision operations against these postulates is left for a sequel paper.

An extended version of maximum entropy is \textit{minimum cross-entropy} (MCE) \citep{k68,c75b}:
\begin{definition}[Minimum Cross-Entropy]
\label{def:MCE}
The `directed divergence' of distribution $c$ from distribution $b$ is defined as
\[
R(c,b) := \sum_{w\in W}c(w)\ln \frac{c(w)}{b(w)}.
\]
$R(c,b)$ is undefined when $b(w)=0$ while $c(w)>0$; when $c(w)=0$, $R(c,b)=0$, because $\lim_{x\to0}\ln(x)=0$.
Given new evidence $\phi\in L^\mathit{prob}$, the distribution $c$ satisfying $\phi$ diverging least from current belief state $b$ is
\[
\argmin_{c\in \Pi,c\Vdash\phi}R(c,b).
\]
\end{definition}
\begin{definition}[$\mathsf{MCI}$]
\label{def:MCI}
Then MCE inference (denoted ($\mathsf{MCI}$)) is defined as
\[
b\:\mathsf{MCI}\:\alpha := \argmin_{
b'\in \Pi,b'\Vdash(\alpha)=1}R(b',b).
\]
\end{definition}

In the following example, we interpret revision as MCE inference.
\begin{example}
\label{ex:9}
Suppose the vocabulary is $\{q,r\}$ and $B_1 = \{(q)\geq 0.6\}$.
Let $\alpha$ be $(q\land\lnot r)\lor(\lnot q\land r)$.
Then
\begin{eqnarray*}
\Pi^{B_1}_\mathit{bnd} &=& \{\{(11,1.0),(10,0.0),(01,0.0),(00,0.0)\},\\
&&\{(11,0.0),(10,1.0),(01,0.0),(00,0.0)\},\\
&&\{(11,0.6),(10,0.0),(01,0.4),(00,0.0)\},\\
&&\{(11,0.6),(10,0.0),(01,0.0),(00,0.4)\},\\ 
&&\{(11,0.0),(10,0.6),(01,0.4),(00,0.0)\},\\ 
&&\{(11,0.0),(10,0.6),(01,0.0),(00,0.4)\}\},
\end{eqnarray*}
\begin{eqnarray*}
(\Pi^{B_1}_\mathit{bnd})^\mathsf{MCI}_\alpha &=& \{\{(11,0),(10,0),(01,1),(00,0)\},\\
&&\{(11,0),(10,1),(01,0),(00,0)\},\\
&&\{(11,0),(10,0.6),(01,0.4),(00,0)\}\}\mbox{ and}
\end{eqnarray*}
$B^\alpha_{1 \mathit{bnd}} = \{(\phi_{11})\leq0$, $(\phi_{00})\leq0\}$.

Note that if we let $B'=\{((q\land\lnot r)\lor(\lnot q\land r))=1\}$, then $\Pi^{B'}=\Pi^{B^\alpha_{1 \mathit{bnd}}}$.
\hfill$\Box$
\end{example}
Recall from Example~\ref{ex:6} that $B'$ included $(q\land\lnot r)\geq0.3$. Hence, in this particular case, combining the boundary belief states approach with $\mathsf{MCI}$ results in a less informative revised belief base than when $\mathsf{GI}$ is used.
The reason for the loss of information might be due to $R(\cdot,\{(11,1.0),(10,0.0),(01,0.0),(00,0.0)\})$ and $R(\cdot,\{(11,0.6),(10,0.0),(01,0.0),(00,0.4)\})$ being undefined: Recall that $R(c,b)$ is undefined when $b(w)=0$ while $c(w)>0$. But then there is no belief state $c$ for which $c\Vdash\alpha$ and $R(\cdot)$ is defined (with these two belief states as arguments).
Hence, there are no revised counterparts of these two belief states in $(\Pi^{B_1}_\mathit{bnd})^\mathsf{MCI}_\alpha$.
We would like to analyse $\mathsf{MCI}$ more within this framework. In particular, in the future, we would like to determine whether a statement like Theorem~\ref{th:main} holds for $\mathsf{MCI}$ too.

In MCE inference, $b$-consistency of evidence $\phi$ is defined as: There exists a belief state $c$ such that $c\Vdash\phi$ and $c$ is \textit{totally continuous} with respect to $b$ (i.e., $b(w) = 0$ implies $c(w) = 0$).
MCE is undefined when the evidence is not $b$-consistent. This is analogous to Bayesian conditioning being undefined for $b(\alpha)=0$. Obviously, this is a limitation of MCE because some belief states may not be considered as candidate revised belief states. Admittedly, we have not searched the literature on this topic due to it being out of the present scope.

As far as we know, imaging for belief change has never been applied to (conditional) probabilistic evidence. Due to issues with many revision methods required to be consistent with prior beliefs, and imaging not having this limitation, it might be worthwhile investigating.

The translation from the set of belief states back to a belief base is a mapping from every belief state to a probability formula. The size of the belief base is thus in the order of $|W^\mathit{perm}|$, where $|W|$ is already exponential in the size of $\Prp$, the set of atoms.
As we saw in several examples in this paper, the new belief base often has a more concise equivalent counterpart. It would be useful to find a way to consistently determine more concise belief bases than our present approach does.

The computational complexity of the process to revise a belief base is at least exponential. This work focused on theoretical issues. If the framework presented here is ever used in practice, computations will have to be optimized.

The following example illustrates how one might deal with strict inequalities.
\begin{example}
\label{ex:8}
Suppose the vocabulary is $\{q,r\}$ and $B_3 = \{(q)> 0.6\}$.
Let $\alpha$ be $(q\land\lnot r)\lor(\lnot q\land r)$.
Let $\epsilon$ be a real number which tends to 0.
Then $\Pi^{B_3}_\mathit{bnd}=$
\begin{align*}
&\{\{(11,1.0),(10,0.0),(01,0.0),(00,0.0)\},\\
&\{(11,0.0),(10,1.0),(01,0.0),(00,0.0)\},\\
&\{(11,0.6+\epsilon),(10,0.0),(01,0.4-\epsilon),(00,0.0)\},\\
&\{(11,0.6+\epsilon),(10,0.0),(01,0.0),(00,0.4-\epsilon)\},\\
&\{(11,0.0),(10,0.6+\epsilon),(01,0.4-\epsilon),(00,0.0)\},\\
&\{(11,0.0),(10,0.6+\epsilon),(01,0.0),(00,0.4-\epsilon)\}\},
\end{align*}
$(\Pi^{B_3}_\mathit{bnd})^\mathsf{GI}_\alpha =$
\begin{align*}
&\{\{(11,0.0),(10,0.5),(01,0.5),(00,0.0)\},\\
&\{(11,0.0),(10,1.0),(01,0.0),(00,0.0)\},\\
&\{(11,0.0),(10,0.3+\epsilon),(01,0.7-\epsilon),(00,0.0)\},\\
&\{(11,0.0),(10,0.6+\epsilon),(01,0.4-\epsilon),(00,0.0)\},\\
&\{(11,0.0),(10,0.8+\epsilon),(01,0.2-\epsilon),(00,0.0)\}\mbox{ and}
\end{align*}
$B^\alpha_{3 \mathit{bnd}} = \{(\phi_{11})\leq0$, $(\phi_{10})\geq0.3+\epsilon$, $(\phi_{01})\leq0.7-\epsilon$, $(\phi_{00})\leq0.0\}$.

Note that if we let $B'=\{((q\land\lnot r)\lor(\lnot q\land r))=1$, $(q\land\lnot r)>0.3\}$, then $\Pi^{B'}=\Pi^{B^\alpha_{3 \mathit{bnd}}}$.
\hfill$\Box$
\end{example}

It has been suggested by one of the reviewers that $\mathsf{GI}$ could be an \textit{affine map} (i.t.o. geometry), thus allowing the proof of Theorem~\ref{th:main} to refer to existing results in the study of affine maps to significantly simplify the proof. The authors are not familiar with affine maps and thus leave investigation of the suggestion to other researchers.

\section{Related Work}

\citet{v99a} proposed the partial probability theory (PTT), which allows probability assignments to be partially determined, and where there is a distinction between probabilistic information based on (i) hard background evidence and (ii) some assumptions. He does not explicitly define the ``constraint language'', however, from his examples and discussions, one can infer that he has something like the language $L^\mathit{PTT}$ in mind:
it contains all formulae which can be formed with sentences in our $L^\mathit{prob}$ in combination with connectives $\lnot, \land$ and $\lor$.
A ``belief state'' in PTT is defined as the quadruple $\langle \Omega,\mathcal{B},\mathcal{A},\mathcal{C}\rangle$, where $\Omega$ is a sample space, $\mathcal{B}\subset L^\mathit{PTT}$ is a sets of probability constraints, $\mathcal{A}\subset L^\mathit{PTT}$ is a sets of assumptions and $\mathcal{C}\subseteq W$ ``represents specific information concerning the case at hand'' (an observation or evidence).\footnote{\citet{v99a}'s ``belief state'' would rather be called and \textit{epistemic state} or \textit{knowledge structure} in our language.} Our epistemic state can be expressed as a restricted PTT ``belief state'' by letting $\Omega=W$, $\mathcal{B}=B$, $\mathcal{A}=\emptyset$ and $\mathcal{C}=\{w\in W\mid w\Vdash\alpha\}$, where $B$ is a belief base and $\alpha$ is an observation in our notation.

\citet{v99a} mentions that he will only consider conditioning where the evidence does not contradict the current beliefs. He defines the set of belief states corresponding to the conditionalized PPT ``belief state'' as $\{b(\cdot\mid C)\in\Pi\mid b\in\Pi^{\mathcal{B}\uni\mathcal{A}}, b(C)>0\}$.
%\footnote{Voorbraak's belief states are not to be confused with our belief states -- his are epistemic states, ours are probability functions.}
In our notation, this corresponds to $\{(b\:\mathsf{BC}\:\alpha)\in\Pi\mid b\in\Pi^B, b(\alpha)>0\}$, where $\alpha$ corresponds to $C$.

\citet{v99a} proposes \textit{constraining} as an alternative to conditioning: 
Let $\phi\in L^\mathit{prob}$ be a probability constraint.
In our notation, constraining $\Pi^B$ on $\phi$ produces $\Pi^{B\uni\{\phi\}}$.

Note that expanding a belief set reduces the number of models (worlds) and expanding a PPT "belief state" with extra constraints also reduces the number of models (belief states / probability functions).
\begin{quote}
In the context of belief sets, it is possible to obtain any belief state from the ignorant belief state by a series of expansions. In PPT, constraining, but not conditioning, has the analogous property. This is one
of the main reasons we prefer to constraining and not conditioning to be the probabilistic version of expansion. \citep[p.~4]{v99a}
\end{quote}

But Voorbraak does not address the issue that $C$ and $\phi$ are different kinds of observations, so constraining, as defined here, cannot be an alternative to conditioning.
$C$ cannot be used directly for constraining and $\phi$ cannot be used directly for conditioning.

W.l.o.g., we can assume $C$ is represented by $\alpha$.
If we take $b\:\mathsf{GI}\:\alpha$ to be an expansion operation whenever $b(\alpha)>0$, then one might ask, Is it possible to obtain any belief base $B'$ from the ignorant belief base $B=\emptyset$ by a series of expansions, using our approach?
The answer is, No. For instance, there is no observation or series of observations which can change $B = \{\}$ into $B' = \{(q)\geq 0.6\}$.
But if we \textit{were} to allow sentences (constraints) in $L^\mathit{prob}$ to be observations, then we \textit{could} obtain any $B'$ from the ignorant $B$.

\citet{gh98} investigate what ``update'' (incorporation of an observation with current beliefs, such that the observation does not contradict the beliefs) means in a framework where beliefs are represented by a set of belief states. They state that the main purpose of their paper is to illustrate how different the set-of-distributions framework can be, ``technically'', from the standard single-distribution framework.
They propose six postulates characterizing what properties an update function should have. They say that some of the postulates are obvious, some arguable and one probably too strong. Out of seven (families of) update functions only the one based on conditioning ($\mathit{Upd}_\mathit{cond}(\cdot)$) and the one based on constraining ($\mathit{Upd}_\mathit{constrain}(\cdot)$) satisfy all six postulates, where $\mathit{Upd}_\mathit{cond}(\Pi^B,\alpha) := \{(b\:\mathsf{BC}\:\alpha)\in\Pi\mid b\in\Pi^B, b(\alpha)>0\}$ and where they interpret Voorbraak's (\citeyear{v99a}) constraining as $\mathit{Upd}_\mathit{constrain}(\Pi^B,\alpha) := \{b\in\Pi^B\mid b(\alpha)=1\}$.
\citet{gh98} do not investigate the case when an observation must be incorporated while it is (possibly) inconsistent with the old beliefs (i.e., revision).

\citet{k01a} develops a new perspective of probabilistic belief change. Based on the ideas of \citet{agm85} and \citet{km91} (KM),
the operations of revision and update, respectively, are investigated within a probabilistic framework.
She employs as basic knowledge structure a belief base $(b,\mathcal{R})$, where $b$ is a probability distribution (belief state) of background knowledge and $\mathcal{R}$ is a set of probabilistic conditionals of the form $A\leadsto B[x]$ meaning `The probability of $B$, given $A$, is $x$.
A universal inference operation -- based on the techniques of optimum entropy -- is introduced as an ``adequate and powerful method to realize probabilistic belief change''.

By having a belief state available in the belief base, minimum cross-entropy can be used.
The intention is then that an agent with belief base $(b,\mathcal{T})$ should always reason w.r.t. belief state $b^\mathcal{T} := \argmin_{c\in \Pi,c\Vdash\mathcal{T}}R(c,b)$.
\citet{k01a} then defines the probabilistic belief revision of $(b,\mathcal{R})$ by evidence $\mathcal{S}$ as $(b,\mathcal{R}\uni\mathcal{S})$.
%, where the agent should now reason w.r.t. $\argmin_{c\in \Pi,c\Vdash\mathcal{R}\uni\mathcal{S}}R(c,b)$.
%
And the probabilistic belief update of $(b,\mathcal{R})$ by evidence $\mathcal{S}$ is defined as $(b^\mathcal{R},\mathcal{S})$.\footnote{This is a very simplified version of what she presents. Please refer to the paper for details.}
She distinguishes between revision as a knowledge adding process, and updating as a change-recording process.
\citet{k01a} sets up comparisons of maximum cross-entropy belief change with AGM revision and KM update.
Cases where, for update, new information $\mathcal{R}$ is inconsistent with the prior distribution $b$, or, for revision, is inconsistent with $b$ or the context $\mathcal{R}$, are not dealt with \citep[p.~399,~400]{k01a}.

Having a belief state available for modification when new evidence is to be adopted is quite convenient. As \citet{v99a} argues, however, an agent's ignorance can hardly be represented in an epistemic state where a single belief state must always be chosen.

The reader may also refer to a later paper \citep{k08} in which many of the results of the work just reviewed are generalized to belief bases of the form $(\Psi,\mathcal{R})$, where $\Psi$ denotes a general \textit{epistemic state}. In that paper, she considers two instantiations of $\Psi$, namely as a \textit{probability distribution} and as an \textit{ordinal conditional function} (first introduced by \citet{s88}).

\cite{yl08} propose a probabilistic revision operation for imprecise probabilistic beliefs in the framework of Probabilistic Logic
Programming (PLP). New evidence may be a probabilistic (conditional) formula and needs not be consistent with the original beliefs.
Revision via imaging (e.g., $\mathsf{GI}$) also overcomes this consistency issue.

Essentially, their \textit{probabilistic epistemic states} $\Psi$ are induced from a PLP program which is a set of formulae, each formula having the form $(\psi\mid\phi)[l,u]$, meaning that the probability of the conditional $(\psi\mid\phi)$ lies in the interval $[l,u]$.

The operator they propose has the characteristic that if an epistemic state $\Psi$ represents a single probability distribution, revising collapses to Jeffrey's rule and Bayesian conditioning.

They mention that it is required that the models (distributions) of $\Psi$ is a convex set. There might thus be an opportunity to employ their revision operation on a representative set of boundary distributions as proposed in this paper.

\section{Conclusion}

In this paper, we propose an approach how to generate a new probabilistic belief base from an old one, given a new piece of non-probabilistic information, where a belief base is a finite set of sentences, each sentence stating the likelihood of a proposition about the world. In this framework, an agent's belief base represents the set of belief states compatible with the sentences in it. In this sense, the agent is able to represent its knowledge \textit{and} ignorance about the true state of the world.

We used a version of the so-called \textit{imaging} approach to implement the revision operation.

Two methods were proposed: revising a finite set of `boundary belief states' and revising a least biased belief state. We focussed on the former and showed that the latter gives different results.

%Future research directions and related work were discussed in detail.

There were two main contribution of this paper. The first was to prove that the set of belief states satisfying $B_\mathit{new}$ is exactly those belief states satisfying the original belief base, revised. The second was to uncover an interesting conflict in the results of the two belief base revision methods.
It is worth further understanding the reasons behind such a difference, as such an investigation could give more insight about the mechanisms behind the two methods and indicate possible pros and cons of each.

\section{Acknowledgements}
The work of Giovanni Casini has been supported by the Fonds National de la Recherche, Luxembourg, and cofunded by the Marie Curie Actions of the European Commission (FP7-COFUND) (AFR/9181001).

\bibliographystyle{aaai}
\bibliography{references}

\end{document}